\pdfoutput=1
\documentclass[runningheads]{llncs}
\usepackage{graphicx}
\usepackage{amsmath,amssymb} 
\usepackage{color}
\usepackage[ruled,vlined]{algorithm2e}
\usepackage{multirow}
\usepackage{mathrsfs}
\usepackage{epsfig}
\usepackage{epstopdf}
\usepackage[width=122mm,left=12mm,paperwidth=146mm,height=193mm,top=12mm,paperheight=217mm]{geometry}
\begin{document}
\pagestyle{headings}
\mainmatter

\title{Fast Subspace Clustering Based on the Kronecker Product} 

\titlerunning{Fast Subspace Clustering Based on the Kronecker Product}

\authorrunning{Zhou \emph{et al.}}

\author{Lei Zhou$^1$, Xiao Bai$^1$, Xianglong Liu$^1$, Jun Zhou$^2$, and Hancock Edwin$^3$}

\institute{$^1$Beihang University $^2$Griffith University $^3$University of York, UK}

\maketitle

\begin{abstract}
Subspace clustering is a useful technique for many computer vision applications in which the intrinsic dimension of high-dimensional data is often smaller than the ambient dimension. Spectral clustering, as one of the main approaches to subspace clustering, often takes on a sparse representation or a low-rank representation to learn a block diagonal self-representation matrix for subspace generation. However, existing methods require solving a large scale convex optimization problem with a large set of data, with computational complexity reaches $\mathcal{O}(N^3)$ for $N$ data points. Therefore, the efficiency and scalability of traditional spectral clustering methods can not be guaranteed for large scale datasets. In this paper, we propose a subspace clustering model based on the Kronecker product. Due to the property that the Kronecker product of a block diagonal matrix with any other matrix is still a block diagonal matrix, we can efficiently learn the representation matrix which is formed by the Kronecker product of $k$ smaller matrices. By doing so, our model significantly reduces the computational complexity to $\mathcal{O}(kN^{3/k})$. Furthermore, our model is general in nature, and can be adapted to different regularization based subspace clustering methods. Experimental results on two public datasets show that our model significantly improves the efficiency compared with several state-of-the-art methods. Moreover, we have conducted experiments on synthetic data to verify the scalability of our model for large scale datasets.
\end{abstract}

\section{Introduction}

In many computer vision applications, such as face recognition~\cite{basri2003lambertian,Liu2013Robust}, texture recognition~\cite{Peng2017Subspace} and motion segmentation~\cite{Elhamifar2013Sparse,kanatani2001motion}, visual data can be well characterized by subspaces. Moreover, the intrinsic dimension of high-dimensional data is often much smaller than the ambient dimension~\cite{vidal2011subspace}. This has motivated the development of subspace clustering techniques which simultaneously cluster the data into multiple subspaces and also locate a low-dimensional subspace for each class of data.

Many subspace clustering algorithms have been developed during the past decade, including algebraic~\cite{costeira1998multibody,vidal2005generalized}, iterative~\cite{agarwal2004k,bradley2000k}, statistical~\cite{rao2008motion,tipping1999mixtures}, and spectral clustering methods~\cite{Elhamifar2013Sparse,Liu2013Robust,Lu2014Correlation,Lu2012Robust,Patel2013Latent,Patel2014Kernel,Peng2017Subspace,Peng2015Subspace,Peng2016Feature,you2016scalable}. Among these approaches, spectral clustering methods have been intensively studied due to their simplicity, theoretical soundness, and empirical success. These methods are based on the self-expressiveness property of data lying in a union of subspaces. This states that each point in a subspace can be written as a linear combination of the remaining data points in that subspace. Two typical methods falling into this category are sparse subspace clustering (SSC)~\cite{Elhamifar2013Sparse} and low-rank representation (LRR)~\cite{Liu2013Robust}. SSC uses the $\ell_1$ norm to encourage the sparsity of the self-representation coefficient matrix. LRR uses nuclear norm minimization to make the coefficient matrix low-rank.

Motivated by SSC and LRR, some self-representation based methods have been developed, which use different regularization terms on the coefficient matrix. For example, least squares regression (LSR)~\cite{Lu2012Robust} uses $\ell_2$ regularization on the coefficient matrix. Correlation adaptive subspace segmentation (CASS)~\cite{Lu2014Correlation} uses a mixture of $\ell_1$ and $\ell_2$ regularization. Low-rank sparse subspace clustering (LRSSC)~\cite{Zhuang2012Non} and non-negative low-rank sparse (NNLRS)~\cite{Wang2013Provable} construct regularization term as a blend of $\ell_1$ and the nuclear norms. Because the nuclear norm does not achieve the accuracy in estimating the rank of real world data, subspace clustering with log-determinant approximation (SCLA)~\cite{Peng2015Subspace} replaces the nuclear norm used in LRR by non-convex rank approximations. Feature selection embedded subspace clustering (FSC)~\cite{Peng2016Feature} reveals that not all features are equally important in the recovery of the low-dimensional subspaces. With feature selection both nuclear norm and non-convex rank approximations may give enhanced performance. Latent space sparse subspace clustering (LS3C)~\cite{Patel2013Latent} seeks a linear projection of the data and learns a sparse representation in the projected latent low-dimensional space.

Despite the fact that SSC, LRR and their variants have achieved encouraging results in practice, they have critical limitations. In these approaches, the key idea is to learn a coefficient matrix which denotes the correlation between the data points. As the size of the coefficient matrix is $N^2$ for $N$ data points, the SVD decomposition operation for solving the coefficient matrix has computational complexity of $\mathcal{O}(N^3)$. This is time consuming when the size of the data is large, so the efficiency of these approaches can not be guaranteed. Experiments in~\cite{you2016scalable} and also in this paper show that some existing methods need to run for several hours on a normal computer when the number of test data reaches $10^4$, which constrains the feasibility of these methods.

To overcome this limitation, we propose an efficient subspace clustering model based on the Kronecker product which achieves a significant reduction of computational complexity over quadratic~\cite{van1993approximation}. Using the fact that each data point in a subspace can be written as a linear combination of all other points in that subspace, we can obtain points lying in the same subspace by learning the sparsest combination. Hence, in our model, we first learn a self-representation coefficient matrix formed by the Kronecker product of a series of small sparse matrices. Then we can constract a similarity matrix based on the coefficient matrix. Finally, a segmentation of the data can be obtained by spectral clustering on the similarity matrix.

The main contributions of this paper are as follows:
\begin{enumerate}
\item We propose an efficient subspace clustering model based on the Kronecker product. Our model uses the Kronecker product of a set of small matrices to build the self-representation coefficient matrix, which leads to a significant reduction of space and computational complexity.
\item Our model is adaptive for different regularization based subspace clustering methods~\cite{Elhamifar2013Sparse,Liu2013Robust,Peng2015Robust,Peng2017Subspace}. And we theoretically prove that the Kronecker product approximation in our model has good adaptivity.
\item Experimental results on large scale synthetic data and real world public datasets show that our method leads to a significant improvement in the clustering efficiency compared with the state-of-the-art methods while also achieving competitive accuracy.
\end{enumerate}
\section{Related Work}

In this section, we review some classical and state-of-the-art methods for subspace clustering.

\subsection{Sparse Subspace Clustering (SSC)}

Given a data matrix $X=[x_i\in\mathbb{R}^D]_{i=1}^N$ that contains $N$ data points drawn from $n$ subspaces $\{S_i\}_{i=1}^n$. SSC~\cite{Elhamifar2013Sparse} aims to find a sparse representation matrix $C$ showing the mutual similarity of the points, i.e., $X=XC$. Since each point in $S_i$ can be expressed in terms of the other points in $S_i$, such a sparse representation matrix $C$ always exists. The SSC algorithm finds $C$ by solving the following optimization problem:
\begin{equation}\label{ssc}
\underset{C}{\min} \ \|C\|_1 \quad \text{s.t.} \ X=XC, \ diag(C)=0,
\end{equation}
where $diag(C)=0$ eliminates the trivial solution.

\subsection{Low-Rank Representation (LRR)}
As pointed out in~\cite{Liu2013Robust}, SSC finds the sparsest representation of each data vector individually. There is no global constraint on its solution, so the SSC method may be inaccurate at capturing the global structures of data. Liu \emph{et al.}~\cite{Liu2013Robust} proposed that low rank can be a more appropriate criterion. Similar to SSC, LRR aims to find a low-rank representation of $X$ by solving the following optimization problem, since the nuclear norm $\|C\|_*$ is the best convex approximation of $rank(W)$ over the unit ball of matrices:
\begin{equation}
\underset{C}{\min} \ \|C\|_* \quad \text{s.t.} \ X=XC,
\end{equation}
where $\|C\|_*$ is the sum of the singular values of $C$.

\subsection{Thresholding Ridge Regression (TRR)}

The SSC and LRR methods solve the robust subspace clustering problem by removing the errors from the original data space and obtaining a good affinity matrix based on a clean dataset. Thus they need prior knowledge of the structure of the errors, which usually is unknown in practice. Peng \emph{et al.}~\cite{Peng2015Robust} proposed a robust subspace clustering method which overcomes this limitation by eliminating the effect of errors from the projection space with a model based on thresholding ridge regression (TRR):
\begin{equation}
\underset{C}{\min} \ \|X-XC\|_F^2+\lambda \|C\|_F^2 \quad \text{s.t.} \ diag(C)=0,
\end{equation}
where $\lambda$ is a balancing parameter and small values in $C$ are truncated to zero by thresholding.

Based on TRR, a 2D nonlinear variance regularized ridge regression (NVR3)~\cite{Peng2017Subspace} was proposed to directly use 2D data, and thus the spatial information is maximally retained.

Each of these related works learns the coefficient matrix $C$ with computational complexity $\mathcal{O}(N^3)$. This has limited the scalability of these methods on large scale datasets. Due to the effectiveness of the Kronecker product in reducing the computational complexity of matrix operations, we present a Kronecker product based subspace clustering model which can significantly improve the efficiency of the existing methods.

\section{Kronecker Product Based Model}\label{subsection1}

In this section, we describe our subspace clustering model based on the Kronecker product and develop an associated optimization scheme.

We first introduce the Kronecker product. Let $A\in\mathbb{R}^{m\times n}$, $B\in\mathbb{R}^{p\times q}$, the Kronecker product of matrices $A$ and $B$ is $A\otimes B\in\mathbb{R}^{mp\times nq}$ which is defined as:
$$
A\otimes B
=\begin{bmatrix}
a_{11}\times B  & \cdots\ &a_{1n}\times B\\
 \vdots & \ddots  & \vdots  \\
 a_{m1}\times B & \cdots\ & a_{mn}\times B\\
\end{bmatrix},
$$
where $a_{ij}$ is the element at the $i$-th row and $j$-th column of $A$.
\subsection{Problem Statement and Formulation}

Let $X=[x_i\in\mathbb{R}^D]_{i=1}^N\in\mathbb{R}^{D\times N}$ be a collection of data points drawn from different subspaces. The goal of subspace clustering is to find the segmentation of the points according to the subspaces. Based on the self-expressiveness property of data lying in a union of subspaces, i.e., each point in a subspace can be written as a linear combination of the remaining points in that subspace, we can obtain points lying in the same subspace by learning the sparsest combination. Therefore, we need to learn a sparse self-representation coefficient matrix $C$, where $X=XC$, and $C_{ij}=0$ if the $i$-th and $j$-th data points are from different subspaces.

As our model aims to reduce the computational complexity with data size $N$, we rewrite $X$ as $X=\{y_i^T\in\mathbb{R}^N\}_{i=1}^D$, where $T$ denotes matrix transpose and $y_i\in\mathbb{R}^{N\times 1}$ is the $i$-th dimension of the data points. Without loss of generality, we assume that the self-representation matrix is formed by the Kronecker product of two smaller matrices $C_1$ and $C_2$, where $C_1\in\mathbb{R}^{p_1\times q_1}$ and $C_2\in\mathbb{R}^{p_2\times q_2}$, where $p_1p_2=N$ and $q_1q_2=N$.  Here we use the important property that the Kronecker product of a block diagonal matrix with any other matrix is still a block diagonal matrix (as shown in Figure~\ref{example}). We follow~\cite{Peng2015Robust} to minimize the loss of self-representation. The optimization problem can be written as:
\begin{equation}\label{opt1}
\underset{C_i}{\min} \ \|X-X(C_1\otimes C_2)\|_F^2+\lambda\|C_1\otimes C_2\|_F^2,
\end{equation}
where $\lambda$ is a balancing parameter, and $\|.\|_F$ is the Frobenius norm.

\begin{figure}[t]
\centering
\includegraphics[width=0.9\linewidth]{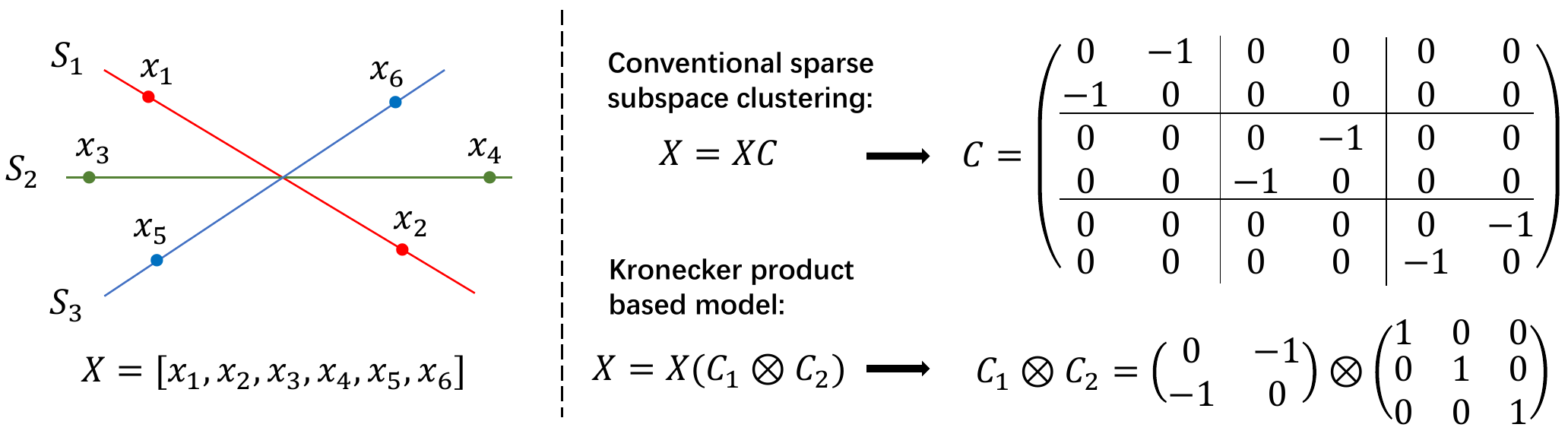}
\caption{Left: Three 1D subspaces in $\mathbb{R}^2$ with normalized data points. Right: The solutions of conventional sparse subspace clustering method (upper) and our Kronecker product based model (lower). As shown, the space and computational complexity of our model achieve significant reduction compared with conventional method.}
\label{example}
\end{figure}

\subsection{Optimization}

We solve problem~(\ref{opt1}) by updating each small matrix at a time, while keeping the other one fixed. Considering updating $C_1$, while $C_2$ fixed, we start by rewriting $\|X-X(C_1\otimes C_2)\|_F^2$ as:
\begin{equation}\label{opt2}
\begin{aligned}
&\|X-X(C_1\otimes C_2)\|_F^2 \\
=&tr((X-X(C_1\otimes C_2))^T(X-X(C_1\otimes C_2)))\\
=&\|X\|_F^2-2tr(X(C_1\otimes C_2)X^T)+tr(X(C_1\otimes C_2)(X(C_1\otimes C_2))^T).
\end{aligned}
\end{equation}
Since $\|X\|_F^2$ is a constant, let
$$
\Phi=-2tr(X(C_1\otimes C_2)X^T)+tr(X(C_1\otimes C_2)(X(C_1\otimes C_2))^T),
$$
then, the problem that minimizing $\|X-X(C_1\otimes C_2)\|_F^2$ is equivalent to minimizing $\Phi$.

According to the block property of Kronecker product~\cite{van2000ubiquitous}:
$$a^T(C_1\otimes C_2)=(vec(C_2^TM_{p_2,p_1}(a)C_1))^T,$$
where $a\in\mathbb{R}^N$ and $vec(X)$ forms a vector by column-wise stacking of the matrix $X$ into a vector, and $M_{p_2,p_1}(a)$ reshapes a $p_1p_2=N$ dimensional vector $a$ to a $p_2\times p_1$ matrix by extracting column from the vector $a$. Then
\begin{equation}
\begin{aligned}
\Phi=&\sum_{i=1}^D(-2y_i^T(C_1\otimes C_2)y_i+y_i^T(C_1\otimes C_2)(y_i^T(C_1\otimes C_2))^T)\\
=&\sum_{i=1}^D(-2(vec(C_2^TM_{p_2,p_1}(y_i)C_1))^Ty_i\\
&+(vec(C_2^TM_{p_2,p_1}(y_i)C_1))^Tvec(C_2^TM_{p_2,p_1}(y_i)C_1)).
\end{aligned}
\end{equation}
Let $H_i=C_2^TM_{p_2,p_1}(y_i)$, $G_i=M_{q_2,q_1}(y_i)$. Then, using the property of trace that $tr(ABC)=tr(BCA)$ and $tr(A^T)=tr(A)$,
\begin{equation}
\begin{aligned}
\Phi=&\sum_{i=1}^D(-2tr((H_iC_1)^TG_i)+tr((H_iC_1)^TH_iC_1))\\
=&\sum_{i=1}^D(-2tr(H_iC_1G_i^T)+tr((H_iC_1)^TH_iC_1))\\
=&\sum_{i=1}^D(\|G_i-H_iC_1\|_F^2-\|G_i\|_F^2).
\end{aligned}
\end{equation}

Since $\|G_i\|_F^2$ is a constant, the optimization objective function of $C_1$ can be written as:
\begin{equation}\label{opt4}
\underset{C_1}{\min} \ \|G-HC_1\|_F^2+\lambda\|C_1\|_F^2
\end{equation}
where $H=\sum_{i=1}^DH_i$, $G=\sum_{i=1}^DG_i$. Eq.~(\ref{opt4}) is a well known ridge regression problem~\cite{Arthur1970Ridge} whose optimal solution is $C_1=(H^TH+\lambda I)^{-1}H^TG$. We can solve $C_2$ in a similar manner to $C_1$, when $C_1$ is fixed.
As $H\in\mathbb{R}^{q_2\times p_1}$, $G\in\mathbb{R}^{q_2\times q_1}$ and $p_1p_2=N,q_1q_2=N$, the computational complexity for this solution is $\mathcal{O}(2N^{3/2})$.

When the number of small matrices is $k$, we can also solve it by updating one small matrix at a time, while keeping the remaining matrices fixed. In this situation, the problem is the same as $k=2$ solved above. As $\prod_{i=1}^kp_i=N$, $\prod_{i=1}^kq_i=N$, then the computational complexity of the whole optimization is $\mathcal{O}(kN^{3/k})$.

We have obtained the optimal solution of self-representation coefficient matrix $C=\otimes_{i=1}^kC_i$, where $C_{ij}=0$ if the $i$-th and $j$-th data points are from different subspaces. Hence, the affinity matrix $W$ can be defined as $W=|C|+|C|^T$, where $|C|$ denotes the absolute value matrix of $C$. Then the segmentation of the data $X$ in different subspaces can be obtained by applying a spectral clustering algorithm to the affinity matrix $W$. The whole Kronecker product based subspace clustering model is summarized in Algorithm~\ref{algo1}.

\begin{algorithm}[t]
\label{algo1}
 \SetAlgoLined
 \KwIn{A set of data points $X=\{x_{i}\}_{i=1}^{N}$, the number of subspaces $n$, the number of small matrices $k$ and the balance parameter $\lambda$.}
 \textbf{Steps:}\\
 1. Learn the small matrices $C_1, C_2, \cdots, C_k$.\\
 \For{$i=1,...,k$}{
 Fix $C_1, \cdots, C_{i-1}, C_{i+1}, \cdots, C_k$, update $C_i$.
 Optimize Eq.~(\ref{opt4}), estimate $C_i$ by ridge regression solution.\\
 }
 2. Calculate the self-representation coefficient matrix $C$ by the Kronecker product of small matrices, $C=\otimes_{i=1}^kC_i$.\\
 3. Construct an affinity matrix by $W=|C|+|C|^T$.\\
 4. Calculate the Laplacian matrix $L$ of $W$. \\
 5. Calculate the eigenvector matrix $V$ of $L$ corresponding to its $n$ smallest nonzero eigenvalues.\\
 6. Perform k-means clustering algorithm on the rows of $V$.\\
 \noindent\KwOut{The clustering result of $X$.}
\caption{Subspace Clustering Based on Kronecker Product.}
\end{algorithm}

\section{Theoretical Analysis}
In this section, we give a theoretical analysis of our Kronecker product based model, including a) the adaptivity on different regularizations, b) theoretical convergence analysis, c) complexity analysis.

\subsection{Adaptivity on Different Regularizations}
Since many self-representation based methods use different regularizations on the coefficient matrix, we show that our model can be applied to a variety of different regularizations. We refer to our subspace clustering method described in Section~\ref{subsection1} as KrTRR (Kronecker product based TRR). It utilizes the Frobenius norm to regularize the coefficient matrix. In Eq.~(\ref{opt4}), we simplify the sparsity constraint from $\|C_1\otimes C_2\|_F^2$ to $\|C_1\|_F^2$, using the Kronecker product lemma:
\begin{lemma}
Let $C=C_1\otimes C_2$, then $\|C\|_F^2=\|C_1\|_F^2\|C_2\|_F^2$.
\end{lemma}
\begin{proof}
Assume $C^{ij}$ is the $i$-th column $j$-th row element in $C$, $C_1\in\mathbb{R}^{m\times n}$, $C_2\in\mathbb{R}^{p\times q}$, $C\in\mathbb{R}^{mp\times nq}$. Then
$\|C\|_F^2=\|C_1\otimes C_2\|_F^2=\sum_{i=1}^m\sum_{j=1}^n\|C_1^{ij}C_2\|_F^2=\sum_{i=1}^m\sum_{j=1}^n(C_1^{ij})^2\|C_2\|_F^2=\|C_1\|_F^2\|C_2\|_F^2$.
\end{proof}

Here we introduce two additional Kronecker product lemmas to show that our model can be applied to alternative regularizations.
\begin{lemma}
Let $C=C_1\otimes C_2$, then $\|C\|_1=\|C_1\|_1\|C_2\|_1$.
\end{lemma}
\begin{proof}
Assume $C^{ij}$ is the $i$-th column $j$-th row element in $C$, $C_1\in\mathbb{R}^{m\times n}$, $C_2\in\mathbb{R}^{p\times q}$, $C\in\mathbb{R}^{mp\times nq}$. Then
$\|C\|_1=\|C_1\otimes C_2\|_1=\sum_{i=1}^m\sum_{j=1}^n\||C_1^{ij}|C_2\|_1=\sum_{i=1}^m\sum_{j=1}^n|C_1^{ij}|\|C_2\|_1=\|C_1\|_1\|C_2\|_1$.
\end{proof}

\begin{lemma}
Let $C=C_1\otimes C_2$, then $\|C\|_*=\|C_1\|_*\|C_2\|_*$.
\end{lemma}
\begin{proof}
Assume the SVD decompositions of $C_1$ and $C_2$ are $C_1=U_1\Sigma_1 V_1^T$ and $C_2=U_2\Sigma_2 V_2^T$, respectively. Then $\|C_1\|_*$ is the sum of nonzero entries in the diagonal matrix $\Sigma_1$, $\|C_2\|_*$ is the sum of nonzero entries in the diagonal matrix $\Sigma_2$. $C=C_1\otimes C_2=(U_1\Sigma_1 V_1^T)\otimes(U_2\Sigma_2 V_2^T)=(U_1\otimes U_2)((\Sigma_1 V_1^T)\otimes(\Sigma_2 V_2^T))=(U_1\otimes U_2)(\Sigma_1\otimes\Sigma_2)(V_1\otimes V_2)^T$. Because $\Sigma_1\otimes\Sigma_2$ is a diagonal matrix, then the SVD decomposition of $C$ is $C=(U_1\otimes U_2)(\Sigma_1\otimes\Sigma_2)(V_1\otimes V_2)^T$. So that $\|C\|_*$ is the sum of nonzero entries in the diagonal matrix $\Sigma_1\otimes\Sigma_2$ which is the product of the sum of nonzero entries in the diagonal matrix $\Sigma_1$ and $\Sigma_2$. Then $\|C\|_*=\|C_1\|_*\|C_2\|_*$.
\end{proof}

Based on these two lemmas, the $\ell_1$ norm and nuclear norm regularizations on the coefficient matrix $\|\otimes_{i=1}^kC_i\|_1$, $\|\otimes_{i=1}^kC_i\|_*$ can be simplified to $\|C_i\|_1$ and $\|C_i\|_*$ as shown in Eq.~(\ref{opt4}). So we can also utilize the $\ell_1$ norm and nuclear norm on the self-representation coefficient matrix with a manner similar to SSC and LRR, i.e.
\begin{equation}\label{opt5}
\underset{C_i}{\min} \ \|X-X(\otimes_{i=1}^kC_i)\|_F^2+\lambda\|\otimes_{i=1}^kC_i\|_1
\end{equation}
and
\begin{equation}\label{opt6}
\underset{C_i}{\min} \ \|X-X(\otimes_{i=1}^kC_i)\|_F^2+\lambda\|\otimes_{i=1}^kC_i\|_*
\end{equation}

We refer to these two methods as KrSSC and KrLRR. Following~\cite{Peng2017Subspace}, we can preprocess the data by 2DPCA~\cite{yang2004two} to retain the spatial information in the 2D data. Then we can use the KrTRR method to learn the coefficient matrix $C$ as done in~\cite{Peng2017Subspace}. We refer to this method as KrNVR3. The optimization of these variants of the Kronecker product based method are essentially the same as KrTRR.

In summary, we can leverage the Kronecker product to reduce the computational complexity of learning the coefficient matrix with different regularization options, e.g. Frobenius norm, $\ell_1$ norm and nuclear norm. We present four methods KrSSC, KrLRR, KrTRR and KrNVR3 based on different regularizations and compare them with baseline methods in Section~\ref{sec:experiments}.

\subsection{Theoretical Convergence Analysis}

Here, we prove the reliability of Kronecker product approximation using a theoretical convergence analysis.

According to the idea of mathematical induction, we consider the special condition that $k=2$ to approximate a $p^2\times p^2$ matrix $C$ by $A\otimes A$, where $A$ is a $p\times p$ matrix. The matrix $C$ is partitioned into $p^2$ matrices with dimension $p\times p$, i.e.
$$
C=
\begin{bmatrix}
C_{11}  & \cdots\ &C_{1p}\\
 \vdots & \ddots  & \vdots\\
 C_{p1} & \cdots\ & C_{pp}\\
\end{bmatrix}
$$
Let
$$
C^*=[vec(C_{11}),vec(C_{12}),\cdots,vec(C_{pp})]
$$
Then, we can denote the approximate loss function by:
\begin{equation}\label{opt6}
\begin{aligned}
l&=tr(A\otimes A-C)^2\\
&=(tr(A)^2)^2-2a^TC^*a+tr(C)^2
\end{aligned}
\end{equation}
where $a=vec(A)$. Since
\begin{equation}
\begin{aligned}
tr(A\otimes A)C&=(vec(C))^Tvec(A\otimes A)\\
&=(vec(C^*))^T(vec(A)\otimes vec(A))\\
&=(vec(C^*))^Tvec((vec(A)(vec(A))^T)\\
&=trC^*(vec(A))(vec(A))^T\\
&=(vec(A))^TC^*vec(A)\\
&=a^TC^*a
\end{aligned}
\end{equation}

Let $\nu (A)$ be the vector with non-duplicate elements of $vec(A)$ and $a=vec(A)=D\nu(A)$, here $D$ is the duplication matrix. Then, the first differential of $l$ is
\begin{equation}
\begin{aligned}
\text{d}l&=4(tr(A)^2)a^T\text{d}a-4a^TC^*\text{d}a\\
&=4(tr(A)^2)a^TD\text{d}\nu(A)-4a^TC^*D\text{d}\nu(A)
\end{aligned}
\end{equation}
The first derivative is
\begin{equation}
\frac{\partial l}{\partial \nu(A)}=4(tr(A)^2)D^Tvec(A)-4D^TC^*vec(A)
\end{equation}
Then, we obtain the first-order condition
\begin{equation}
tr(A)^2vec(A)=C^*vec(A)
\end{equation}

This is an eigenvalue problem in terms of $C^*$. The vector $a$ minimizing Eq.~(\ref{opt6}) must be proportional to the eigenvector corresponding to the largest eigenvalue of $C^*$. In other words, for an arbitrary matrix with any dimension, we can partition it based on the dimensions of small matrices needed to approximate the large matrix via Kronecker product. moreover, the small matrices always have a convergent solution through the largest eigenvector of the partitioned large matrix. This means that the technique used to approximate the large self-representation matrix by the Kronecker product of small matrices in our model is reliable.

\subsection{Complexity Analysis}
Here we discuss the space memory requirement and computational complexity of our Kronecker product based methods and compare it to the relevant methods in the literature. When the data size is $N$, methods in~\cite{Elhamifar2013Sparse,Liu2013Robust,Peng2015Robust,Peng2017Subspace} need to solve the self-representation coefficient matrix $C$ with the dimension $N\times N$, i.e., the memory space complexity of these methods is $\mathcal{O}(N^2)$. But in our work, we leverage the Kronecker product of a set of small matrices to approximate the self-representation coefficient matrix $C$. When the number of small matrices is $k$, the size of small matrices is $N^{2/k}$. Thus, the space complexity of our methods is $\mathcal{O}(kN^{2/k})$.

For learning process the self-representation coefficient matrix $C$ with size $N^2$, existing methods use a SVD decomposition operation whose computational complexity is $\mathcal{O}(N^3)$. As our methods update one small matrix at a time, and the size of the small matrix is $N^{2/k}$, we achieve $\mathcal{O}(kN^{3/k})$ computational complexity. Since $N^{1/k} \ll N$ when $k>1$, there is significant reduction in both the memory space and computational complexity compared with the existing methods. This efficiency gain is achieved by using the Kronecker product.

\section{Experiments}\label{sec:experiments}

We have conducted three sets of experiments on both real and synthetic datasets to verify the effectiveness of the proposed methods. Several state-of-the-art or classical spectral subspace clustering methods were taken as the baseline algorithms. These included sparse subspace clustering (SSC)~\cite{Elhamifar2013Sparse}, low-rank representation (LRR)~\cite{Liu2013Robust}, thresholding ridge regression (TRR)~\cite{Peng2015Robust}, and nonlinear variance regularized ridge regression (NVR3)~\cite{Peng2017Subspace}. In the experiments, we used the codes provided by the respective authors for computing the self-representation matrix $C$, where the parameters were tuned to give the best clustering accuracy. Then we applied the normalized spectral clustering in~\cite{von2007tutorial} to the affinity matrix $W=|C|+|C|^T$.

\textbf{Evaluation criteria:} we used both the clustering accuracy and running time of the whole clustering process to evaluate the performance of the subspace clustering methods, where the clustering accuracy is calculated as
$$
\text{clustering accuracy}=\frac{\text{\# of correctly classified points}}{\text{total \# of points}} \times 100
$$
In all our experiments, the clustering accuracy and running time were averaged over 10 trials. All experiments were implemented with MATLAB code and ran on a PC with Intel Core-i7 3.6GHz CPU, 32GB RAM.
\setlength{\tabcolsep}{2pt}
\begin{table}[t]
  \centering
  \caption{The average running time (seconds) and clustering accuracy on the CMU PIE database with different number of objects. Each object consists of 170 face images under different illuminations and expressions. '-' denotes that the computational cost is unacceptable for our PC, due to the memory and time limit.}
  \begin{center}
    \begin{tabular}{c|c|c|c|c|c|c|c|c|c|c}
    \hline
    \multirow{2}{*}{No. Objects}&
    \multicolumn{2}{c|}{5 Objects}&\multicolumn{2}{c|}{10 Objects}&\multicolumn{2}{c|}{20 Objects}&\multicolumn{2}{c|}{40 Objects}&\multicolumn{2}{c}{60 Objects}\cr\cline{2-11}
    &Time&Acc.&Time&Acc.&Time&Acc.&Time&Acc.&Time&Acc.\cr
    \hline
    \hline
    SSC&243.6&92.47&1182&89.25&3618&84.31&14502&82.37&-&-\cr
    \textbf{KrSSC}&\textbf{12.7}&91.28&\textbf{26.8}&88.27&\textbf{61.4}&83.86&\textbf{150.2}&81.75&\textbf{274.3}&79.48\cr
    LRR&216.4&94.53&852.5&92.14&2743&89.21&11463&85.47&-&-\cr
    \textbf{KrLRR}&\textbf{9.7}&92.51&\textbf{20.4}&90.72&\textbf{57.2}&88.13&\textbf{145.8}&85.21&\textbf{254.8}&83.65\cr
    TRR&152.7&97.35&548.2&96.05&2167&94.54&8427&91.74&-&-\cr
    \textbf{KrTRR}&\textbf{7.5}&95.21&\textbf{18.3}&94.52&\textbf{52.8}&93.84&\textbf{143.5}&90.23&\textbf{260.1}&87.26\cr
    NVR3&190.5&98.51&624.6&97.51&2536&95.75&11826&93.15&-&-\cr
    \textbf{KrNVR3}&\textbf{11.3}&97.14&\textbf{25.7}&96.26&\textbf{72.4}&93.96&\textbf{180.4}&91.57&\textbf{312.5}&89.15\cr
    \hline
    \end{tabular}
    \end{center}
\label{face}
\end{table}

\subsection{Face Clustering}

As subspaces are commonly used to capture the appearance of faces under varying illuminations, we test the performance of our method on face clustering with the CMU PIE database~\cite{Sim¨C2001¨C8175}. The CMU PIE database contains 41,368 images of 68 people under 13 different poses, 43 different illumination conditions, and 4 different expressions. In our experiment, we used the face images in five near frontal poses (P05, P07, P09, P27, P29). Then each people has 170 face images under different illuminations and expressions. Each image was manually cropped and normalized to a size of $32\times32$ pixels. In each experiment, we randomly picked $n\in\{5, 10, 20, 40, 60\}$ individuals to investigate the performance of the proposed method. For our models, we set the number of small matrices $k=2$ and $\lambda = 0.25$. For different number of objects $n$, we randomly chose $n$ people with 10 trials and took all the images of them as the subsets to be clustered. Then we conducted experiments on all 10 subsets and report the average running time and clustering accuracy with a different number of objects in Table~\ref{face}.

In the original work, SSC, LRR, TRR, and NVR3 all test on a small subset which consists of no more than 1,000 data points. Because of the memory and time limit, these methods can not run on a dataset of size $\mathcal{O}(10^4)$. In our experiment, the data size is in the range of $N\in\{850, 1700, 3400, 6800, 10200\}$, corresponding to 5-60 objects per face.  As shown in Table~\ref{face}, the efficiency of all alternative methods degrades drastically when $N$ increases. When $N>10000$ (60 objects), the space and computational complexity of these methods are unacceptable for our PC. In contrast, the computational time of Kronecker product based methods is significantly lower compared with the corresponding approaches. Our methods can easily handle more than 10,000 data points with an acceptable computing time. Further, we can see from Table~\ref{face} that the Kronecker product based methods also obtain competitive clustering accuracy (down 2 percent at most). This suggests that our model is potentially more suitable than previous methods on large scale dataset for real world applications.

\subsection{Handwritten Digit Clustering}
Database of handwritten digits is also widely used in subspace learning and clustering. We test the proposed methods on handwritten digit clustering with the MNIST dataset~\cite{L1998Gradient}. This dataset contains 10 clusters, including handwritten digits 0-9. Each cluster contains 6,000 images for training and 1,000 images for testing, with a size of $28\times28$ pixels in each image. We used all the 70,000 handwritten digit images for subspace clustering. Different from the experimental settings for face clustering, we fixed the number of clusters $n=10$ and chose different number of data points for each cluster with 10 trials. Each cluster contains $N_i$ data points randomly chosen from corresponding 7,000 images, where $N_i\in\{50, 100, 1000, 3000, 7000\}$, so that the number of points $N\in\{500, 1000, 10000, 30000, 70000\}$. Then we applied all methods on this dataset for comparison. For our models, we set the number of small matrices $k=2$ and $\lambda=0.2$. The average running time and clustering accuracy with different number of data points are shown in Table~\ref{digit}.

It can be seen that the efficiency of KrSSC, KrLRR, KrTRR and KrNVR3 significantly outperform the corresponding baseline methods, which indicates the effectiveness of the Kronecker product method proposed in this paper. Table~\ref{digit} also shows that our method and its variants obtain competitive clustering accuracy compared with the corresponding baseline methods.

\setlength{\tabcolsep}{2pt}
\begin{table}[t]
  \centering
  \caption{The average running time (seconds) and clustering accuracy on the MNIST dataset with different number of data points. The data consists of randomly chosen $N_i\in\{50, 100, 1000, 3000, 7000\}$ images for each of the 10 digits. '-' denotes that the computational cost is unacceptable on our PC due to the memory and time cost.}
  \begin{center}
    \begin{tabular}{c|c|c|c|c|c|c|c|c|c|c}
    \hline
    \multirow{2}{*}{No. Points}&
    \multicolumn{2}{c|}{500}&\multicolumn{2}{c|}{1000}&\multicolumn{2}{c|}{10000}&\multicolumn{2}{c|}{30000}&\multicolumn{2}{c}{70000}\cr\cline{2-11}
    &Time&Acc.&Time&Acc.&Time&Acc.&Time&Acc.&Time&Acc.\cr
    \hline
    \hline
    SSC&152.4&83.36&638.2&82.45&-&-&-&-&-&-\cr
    \textbf{KrSSC}&\textbf{7.3}&81.25&\textbf{18.7}&81.17&\textbf{192.4}&79.42&\textbf{411.5}&76.15&\textbf{683.2}&73.34\cr
    LRR&145.5&85.75&614.8&85.14&-&-&-&-&-&-\cr
    \textbf{KrLRR}&\textbf{7.1}&83.24&\textbf{16.4}&83.20&\textbf{160.8}&81.52&\textbf{384.5}&79.21&\textbf{641.5}&76.53\cr
    TRR&113.2&90.28&476.4&89.78&-&-&-&-&-&-\cr
    \textbf{KrTRR}&\textbf{6.5}&88.95&\textbf{15.8}&88.65&\textbf{168.2}&85.76&\textbf{403.8}&83.26&\textbf{795.6}&81.53\cr
    NVR3&118.5&91.85&531.1&91.28&-&-&-&-&-&-\cr
    \textbf{KrNVR3}&\textbf{8.3}&90.08&\textbf{22.5}&90.14&\textbf{243.6}&86.27&\textbf{627.5}&83.87&\textbf{968.4}&82.41\cr
    \hline
    \end{tabular}
    \end{center}
\label{digit}
\end{table}

\subsection{Large-Scale Experiment}

\setlength{\tabcolsep}{2pt}
\begin{table}[t]
  \centering
  \caption{The average running time (seconds) and clustering accuracy on synthetic dataset with different number of data points. The data consists of randomly chosen $N_i\in\{100, 1000, 2000, 10000, 20000\}$ points for each of the 5 subspaces. '-' denotes that the computational cost is unacceptable for our PC due to the memory and time limit.}
  \begin{center}
    \begin{tabular}{c|c|c|c|c|c|c|c|c|c|c}
    \hline
    \multirow{2}{*}{No. Points}&
    \multicolumn{2}{c|}{500}&\multicolumn{2}{c|}{5000}&\multicolumn{2}{c|}{10000}&\multicolumn{2}{c|}{50000}&\multicolumn{2}{c}{100000}\cr\cline{2-11}
    &Time&Acc.&Time&Acc.&Time&Acc.&Time&Acc.&Time&Acc.\cr
    \hline
    \hline
    SSC&135.4&94.15&1824&93.86&5413&91.05&-&-&-&-\cr
    \textbf{KrSSC}&\textbf{6.2}&92.12&\textbf{53.4}&91.18&\textbf{164.2}&89.73&\textbf{231.5}&85.04&\textbf{285.7}&81.85\cr
    LRR&118.6&95.27&1645&94.57&4853&92.14&-&-&-&-\cr
    \textbf{KrLRR}&\textbf{6.0}&93.24&\textbf{49.3}&92.21&\textbf{152.7}&89.49&\textbf{216.2}&86.03&\textbf{274.3}&82.20\cr
    TRR&89.5&98.85&1627&97.15&5825&95.69&-&-&-&-\cr
    \textbf{KrTRR}&\textbf{5.9}&98.06&\textbf{46.7}&96.53&\textbf{185.3}&95.05&\textbf{250.3}&93.16&\textbf{314.2}&89.06\cr
    NVR3&96.4&99.91&1752&98.61&6024&97.10&-&-&-&-\cr
    \textbf{KrNVR3}&\textbf{6.0}&99.07&\textbf{52.8}&98.11&\textbf{207.5}&96.24&\textbf{260.1}&93.89&\textbf{321.5}&90.62\cr
    \hline
    \end{tabular}
    \end{center}
\label{synthetic}
\end{table}

To verify the scalability of our method on large scale datasets, we also ran experiments on synthetic data. Following~\cite{you2016scalable}, we randomly generated $n=5$ subspaces, each of dimension $d=6$ in an ambient space of dimension $D=9$. Each subspace contains $N_i$ data points randomly generated on the unit sphere, where $N_i\in\{100, 1000, 2000, 10000, 20000\}$, so that the number of points $N\in\{500, 5000, 10000, 50000, 100000\}$. Due to the memory and time limit, SSC, LRR, TRR and NVR3 were run for $N\le 10000$. For our models, $\lambda=0.2$, the number of small matrices $k=2$ for $N\in\{500,5000,10000\}$ and $k=3$ for $N\in\{50000,100000\}$. With different number of sample points, we conducted experiments on all methods and report the average running time and clustering accuracy in Table~\ref{synthetic}.

As shown in Table~\ref{synthetic}, the advantage of our method and its variants over the baseline methods is more marked on large scale datasets. When the dataset size reaches 10,000, the computational running time of the alternate methods under comparison are about two hours each, but our Kronecker product based methods only need a few thousand seconds even for 100,000 data points. From Table~\ref{synthetic}, it is also clear that when $k$ increases from 2 to 3 for $N\in\{50000,100000\}$, the running time decreases significantly. The clustering accuracy can also be guaranteed compared with existing methods. Due to the limitations of memory space and computational complexity, the alternative methods can not be applied to a dataset of larger than 10,000 points. This again suggests that our methods are potentially more suitable for large real world applications.

\setlength{\tabcolsep}{8pt}
\begin{table}[t]
\caption{The average running time and clustering accuracy of our methods with different $k$.}
\begin{center}
\begin{tabular}{c|cccc}
\hline
$k$ &2&3&4&5\\
\hline
\multicolumn{5}{l}{average running time (seconds):}\cr\cline{1-5}
\hline
KrSSC &715.6 &285.7&61.2&25.4\\
KrLRR &682.5&274.3&52.7&20.6\\
KrTRR &755.1&314.2&84.3&31.5\\
KrNVR3 &794.3&321.5&91.6&36.2\\
\hline
\multicolumn{5}{l}{average clustering accuracy:}\cr\cline{1-5}
\hline
KrSSC &83.14&81.85&75.42&67.25\\
KrLRR &84.43&82.20&77.16&68.17\\
KrTRR &90.75&89.06&84.27&73.41\\
KrNVR3 &92.54&90.62&85.34&75.24\\
\hline
\end{tabular}
\end{center}
\label{parametertable}
\end{table}
\subsection{Parameter Sensitivity}
Here, we report experimental results on a synthetic dataset to illustrate the sensitivity of the Kronecker product based methods to parameter variations. As the parameters $k$ (number of the small matrices) and $\lambda$ (the balancing parameter of Eq.~(\ref{opt1})) in our model are both related to the dataset size $N$, we fix $N=100000$. Table~\ref{parametertable} shows the average running time and clustering accuracy of our methods with different $k\in\{2,3,4,5\}$. We can see that when $k$ increases, the running time significantly decreases but with the sacrifice of clustering accuracy. This implies that the number of small matrices $k$ should be determined by the size of dataset with a compromise between efficiency and accuracy. Figure~\ref{parameterfigure} shows the clustering accuracy of our methods with different balance parameter $\lambda$. It is evident that the clustering accuracy is insensitive when $\lambda\in(0.1,0.5)$.

\begin{figure}[t]
\centering
\includegraphics[width=0.45\linewidth]{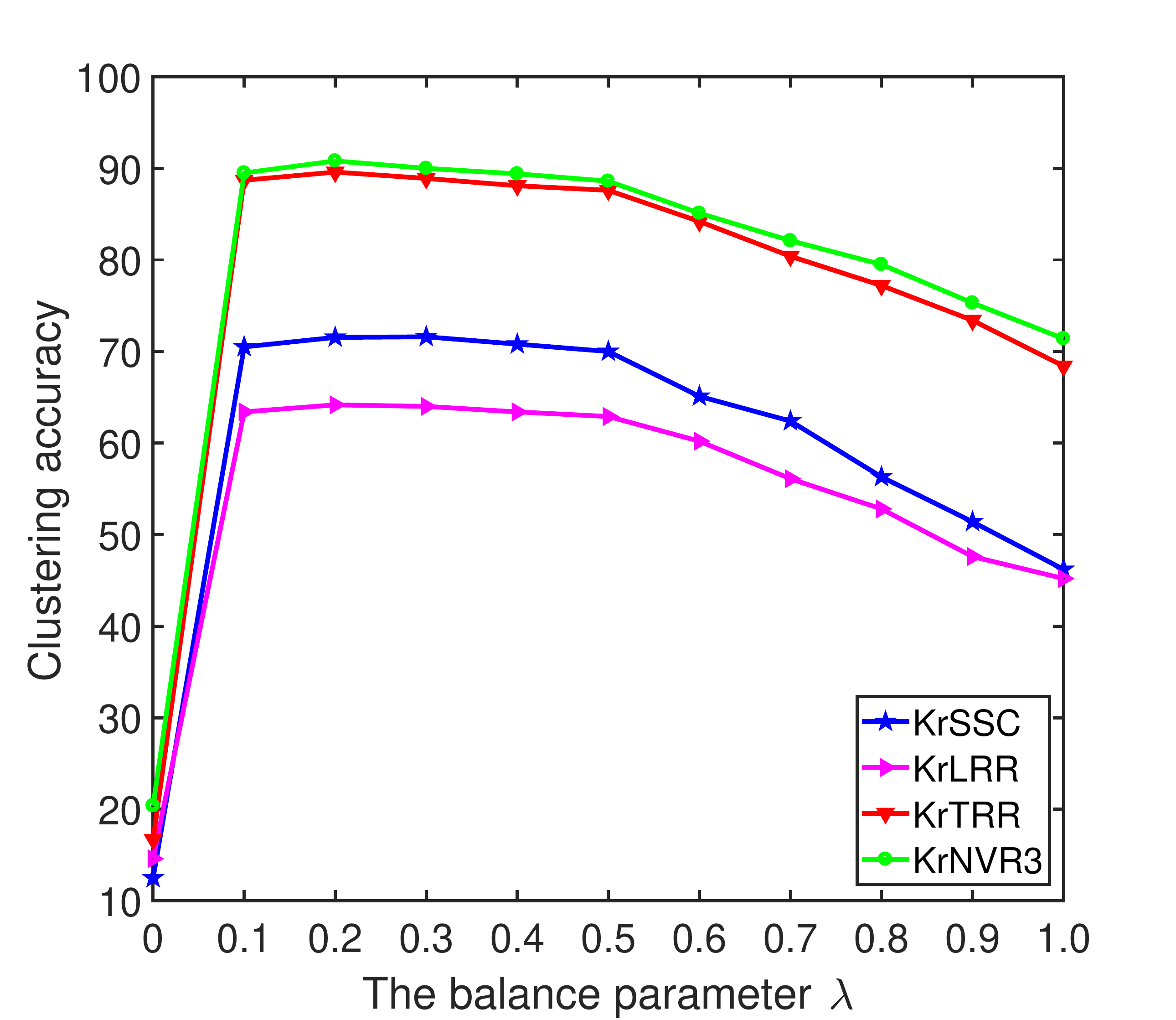}
   \caption{The average clustering accuracy of our methods with different balance parameter $\lambda$.}
\label{parameterfigure}
\end{figure}

\section{Conclusion}

We have presented a fast subspace clustering model based on the Kronecker product. Due to the property that the Kronecker product of a block diagonal matrix and any other matrix is still a block diagonal matrix, we learn the representation matrix of spectral clustering using the Kronecker product of a set of smaller matrices. Thanks to the superiority of the Kronecker product in reducing the computational complexity of matrix operations, the memory space and computational complexity of our methods achieve significant efficiency gain compared with several baseline approaches (SSC, LRR, TRR, and NVR3). We have presented four variants of the Kronecker product based method, namely KrSSC, KrLRR, KrTRR and KrNVR3. Experimental results on face clustering and handwriting digit clustering show that our methods achieve significantly improvement in efficiency compared with the state-of-the-art methods. Moreover, we have presented results on synthetic data which has verified the scalability of our methods on large scale datasets.

\bibliographystyle{splncs03}
\bibliography{egbib}

\begin{thebibliography}{10}
\providecommand{\url}[1]{\texttt{#1}}
\providecommand{\urlprefix}{URL }

\bibitem{agarwal2004k}
Agarwal, P.K., Mustafa, N.H.: K-means projective clustering. In: Symposium on
  Principles of Database Systems. pp. 155--165 (2004)

\bibitem{basri2003lambertian}
Basri, R., Jacobs, D.W.: Lambertian reflectance and linear subspaces. IEEE
  Transactions on Pattern Analysis and Machine Intelligence  25(2),  218--233
  (2003)

\bibitem{bradley2000k}
Bradley, P.S., Mangasarian, O.L.: K-plane clustering. Journal of Global
  Optimization  16(1),  23--32 (2000)

\bibitem{costeira1998multibody}
Costeira, J.P., Kanade, T.: A multibody factorization method for independently
  moving objects. International Journal of Computer Vision  29(3),  159--179
  (1998)

\bibitem{Elhamifar2013Sparse}
Elhamifar, E., Vidal, R.: Sparse subspace clustering: Algorithm, theory, and
  applications. IEEE Transactions on Pattern Analysis and Machine Intelligence
  35(11),  2765--2781 (2013)

\bibitem{Arthur1970Ridge}
Hoerl, A.E., Kennard, R.W.: Ridge regression: Biased estimation for
  nonorthogonal problems. Technometrics  12(1),  55--67 (1970)

\bibitem{kanatani2001motion}
Kanatani, K.i.: Motion segmentation by subspace separation and model selection.
  In: International Conference on Computer Vision. vol.~2, pp. 586--591 (2001)

\bibitem{L1998Gradient}
L¨¦cun, Y., Bottou, L., Bengio, Y., Haffner, P.: Gradient-based learning
  applied to document recognition. Proceedings of the IEEE  86(11),  2278--2324
  (1998)

\bibitem{Liu2013Robust}
Liu, G., Lin, Z., Yan, S., Sun, J., Yu, Y., Ma, Y.: Robust recovery of subspace
  structures by low-rank representation. IEEE Transactions on Pattern Analysis
  and Machine Intelligence  35(1),  171--184 (2013)

\bibitem{Lu2012Robust}
Lu, C.Y., Min, H., Zhao, Z.Q., Zhu, L., Huang, D.S., Yan, S.: Robust and
  efficient subspace segmentation via least squares regression. In: European
  Conference on Computer Vision. pp. 347--360 (2012)

\bibitem{Lu2014Correlation}
Lu, C., Feng, J., Lin, Z., Yan, S.: Correlation adaptive subspace segmentation
  by trace lasso. In: International Conference on Computer Vision. pp.
  1345--1352 (2014)

\bibitem{Patel2013Latent}
Patel, V.M., Van~Nguyen, H., Vidal, R.: Latent space sparse subspace
  clustering. In: International Conference on Computer Vision. pp. 225--232
  (2013)

\bibitem{Patel2014Kernel}
Patel, V.M., Vidal, R.: Kernel sparse subspace clustering. In: International
  Conference on Image Processing. pp. 2849--2853 (2014)

\bibitem{Peng2017Subspace}
Peng, C., Kang, Z., Cheng, Q.: Subspace clustering via variance regularized
  ridge regression. In: Computer Vision and Pattern Recognition (2017)

\bibitem{Peng2015Subspace}
Peng, C., Kang, Z., Li, H., Cheng, Q.: Subspace clustering using
  log-determinant rank approximation. In: International Conference on Knowledge
  Discovery and Data Mining. pp. 925--934 (2015)

\bibitem{Peng2016Feature}
Peng, C., Kang, Z., Yang, M., Cheng, Q.: Feature selection embedded subspace
  clustering. IEEE Signal Processing Letters  23(7),  1018--1022 (2016)

\bibitem{Peng2015Robust}
Peng, X., Yi, Z., Tang, H.: Robust subspace clustering via thresholding ridge
  regression. In: AAAI Conference on Artificial Intelligence. pp. 3827--3833
  (2015)

\bibitem{rao2008motion}
Rao, S.R., Tron, R., Vidal, R., Ma, Y.: Motion segmentation via robust subspace
  separation in the presence of outlying, incomplete, or corrupted
  trajectories. In: Computer Vision and Pattern Recognition. pp. 1--8 (2008)

\bibitem{Sim¨C2001¨C8175}
Sim, T., Baker, S., Bsat, M.: The cmu pose, illumination, and expression (pie)
  database of human faces. Tech. Rep. CMU-RI-TR-01-02, Pittsburgh, PA (January
  2001)

\bibitem{tipping1999mixtures}
Tipping, M.E., Bishop, C.M.: Mixtures of probabilistic principal component
  analyzers. Neural computation  11(2),  443--482 (1999)

\bibitem{van2000ubiquitous}
Van~Loan, C.F.: The ubiquitous kronecker product. Journal of computational and
  applied mathematics  123(1),  85--100 (2000)

\bibitem{van1993approximation}
Van~Loan, C.F., Pitsianis, N.: Approximation with kronecker products. In:
  Linear algebra for large scale and real-time applications, pp. 293--314.
  Springer (1993)

\bibitem{vidal2011subspace}
Vidal, R.: Subspace clustering. IEEE Signal Processing Magazine  28(2),  52--68
  (2011)

\bibitem{vidal2005generalized}
Vidal, R., Ma, Y., Sastry, S.: Generalized principal component analysis (gpca).
  IEEE Transactions on Pattern Analysis and Machine Intelligence  27(12),
  1945--1959 (2005)

\bibitem{von2007tutorial}
Von~Luxburg, U.: A tutorial on spectral clustering. Statistics and computing
  17(4),  395--416 (2007)

\bibitem{Wang2013Provable}
Wang, Y.X., Xu, H., Leng, C.: Provable subspace clustering: when lrr meets ssc.
  In: Advances in Neural Information Processing Systems. pp. 64--72 (2013)

\bibitem{yang2004two}
Yang, J., Zhang, D., Frangi, A.F., Yang, J.y.: Two-dimensional pca: a new
  approach to appearance-based face representation and recognition. IEEE
  Transactions on Pattern Analysis and Machine Intelligence  26(1),  131--137
  (2004)

\bibitem{you2016scalable}
You, C., Robinson, D., Vidal, R.: Scalable sparse subspace clustering by
  orthogonal matching pursuit. In: Computer Vision and Pattern Recognition. pp.
  3918--3927 (2016)

\bibitem{Zhuang2012Non}
Zhuang, L., Gao, H., Lin, Z., Ma, Y., Zhang, X., Yu, N.: Non-negative low rank
  and sparse graph for semi-supervised learning. In: Computer Vision and
  Pattern Recognition. pp. 2328--2335 (2012)

\end{thebibliography}

\end{document}